\documentclass{article}
\usepackage{spconf}
\usepackage[cmex10]{amsmath}
\usepackage{cite}
\usepackage{graphicx,epstopdf}
\usepackage{color}
\usepackage{times}
\usepackage{verbatim}
\usepackage{floatflt}
\usepackage{enumerate}
\usepackage{array}
\usepackage{multicol,afterpage,wrapfig}
\usepackage{tikz}
\usepackage{algorithm}
\usepackage[noend]{algpseudocode}
\makeatletter
\def\BState{\State\hskip-\ALG@thistlm}
\makeatother
\usepackage{amsmath,amsthm,amssymb,eucal}
\usepackage[utf8]{inputenc}
\usepackage{epsfig}
\usepackage{exscale}
\usepackage{latexsym}
\usepackage{verbatim}
\usepackage{amsfonts}
\usepackage{subfigure}
\usepackage{multirow}
\usepackage{cite}
\usepackage{balance}
\usepackage{color}
\usepackage{transparent}
\usepackage{import}
\usepackage{mathtools}
\usepackage{tikz}
\usetikzlibrary{automata,arrows,positioning,calc}
\usepackage{bbm}
\usepackage[printonlyused,withpage]{acronym}
\usepackage{tabu}
\usepackage{longtable}
\usepackage{mathtools}
\usepackage{stfloats}
\usepackage{psfrag}
\usepackage{float}
\usepackage{acronym}  
\usepackage{psfrag}

\acrodef{CCDF}{complementary cumulative distribution function}
\acrodef{CF}{characteristic function}
\acrodef{PPP}{Poisson point process}
\acrodef{RV}{random variable}
\acrodef{i.i.d.}{independent and identically distributed}
\acrodef{PDF}{probability distribution function}
\acrodef{CDF}{cumulative distribution function}
\acrodef{ch.f.}{characteristic function}
\acrodef{AWGN}{additive white Gaussian noise}
\acrodef{SNR}{signal-to-noise ratio}
\acrodef{LRT}{likelihood ratio test}
\acrodef{DRT}{distance ratio test}
\acrodef{GLRT}{generalized likelihood ratio test}
\acrodef{CRLB}{Cram\'{e}r-Rao lower bound}
\acrodef{CRB}{Cram\'{e}r-Rao bound}
\acrodef{ZZLB}{Ziv-Zakai lower bound}
\acrodef{ZZB}{Ziv-Zakai bound}
\acrodef{LOS}{line-of-sight}
\acrodef{ToF}{time-of-flight}
\acrodef{NLOS}{non-line-of-sight}
\acrodef{GDOP}{geometric dilution of precision}
\acrodef{GPS}{Global Positioning System}
\acrodef{FIM}{Fisher information matrix}
\acrodef{PEB}{position error bound}
\acrodef{SPEB}{squared position error bound}
\acrodef{TOA}{time-of-arrival}
\acrodef{TOF}{time-of-flight}
\acrodef{WSN}{wireless sensor network}
\acrodef{MAC}{medium access control}
\acrodef{RSS}{received signal strength}
\acrodef{WAF}{wall attenuation factor}
\acrodef{TDOA}{time difference-of-arrival}
\acrodef{RF}{radiofrequency}
\acrodef{RTT}{round-trip time}
\acrodef{AOA}{angle-of-arrival}
\acrodef{MF}{matched filter}
\acrodef{ED}{energy detector}
\acrodef{ML}{maximum likelihood}
\acrodef{MSE}{mean-square error}
\acrodef{RMSE}{root-mean-square error}
\acrodef{LEO}{localization error outage}
\acrodef{ppm}{part-per-million}
\acrodef{ACK}{acknowledge}
\acrodef{UWB}{Ultrawide bandwidth}
\acrodef{TNR}{threshold-to-noise ratio}
\acrodef{LS}{least squares}
\acrodef{IR-UWB}{impulse radio UWB}
\acrodef{FCC}{Federal Communications Commission}
\acrodef{TH}{time-hopping}
\acrodef{PPM}{pulse position modulation}
\acrodef{MUI}{multi-user interference}
\acrodef{PDP}{power delay profile}
\acrodef{BPZF}{band-pass zonal filter}
\acrodef{SIR}{signal-to-interference ratio}
\acrodef{SINR}{signal-to-interference-plus-noise ratio}
\acrodef{RFID}{radio frequency identification}
\acrodef{WPAN}{wireless personal area network}
\acrodef{WWB}{Weiss-Weinstein bound}
\acrodef{DP}{direct path}
\acrodef{MF}{matched filter}
\acrodef{MMSE}{minimum-mean-square-error}
\acrodef{SBS}{serial backward search}
\acrodef{SBSMC}{serial backward search for multiple clusters}
\acrodef{NBI}{narrowband interference}
\acrodef{WBI}{wideband interference}
\acrodef{INR}{interference-to-noise ratio}
\acrodef{CR}{channel response}
\acrodef{CIR}{channel impulse response}
\acrodef{CR}{channel  response}
\acrodef{RADAR}{radar}
\acrodef{MUR}{Multistatic radar}
\acrodef{JBSF}{jump back and search forward}
\acrodef{HDSA}{high-definition situation-aware}
\acrodef{RRC}{root raised cosine}
\acrodef{ST}{simple thresholding}
\acrodef{BTB}{Bellini-Tartara bound}
\acrodef{P-Max}{$P$-Max}  
\acrodef{MIMO}{multiple-input multiple-output}
\acrodef{MAP}{maximum a posteriori}
\acrodef{FG}{factor graph}
\acrodef{OP}{outage probability}
\acrodef{WED}{wall extra delay}
\acrodef{RMS}{root mean square}
\acrodef{SPAWN}{sum-product algorithm over a wireless network}
\acrodef{MDD}{minimum distance distribution}
\acrodef{MAP}{maximum a posteriori probability}
\acrodef{SAP}{small cell access point}
\acrodef{UE}{user equipment}
\acrodef{MBS}{macro cell base station}
\acrodef{UER}{\ac{UE} Relay}
\acrodef{D2D}{device-to-device}
\acrodef{MBS}{macro base station}
\acrodef{CSI}{channel state information}
\acrodef{OGR}{outage guard region}
\acrodef{FUR}{feasible UER region}
\acrodef{EHR}{energy harvesting region}
\acrodef{EH}{energy harvesting}
\acrodef{D2D-EHSN}{D2D communication provided \ac{EH} small cell network}
\acrodef{D2D-EHHN}{D2D communication provided \ac{EH} heterogeneous network}
\acrodef{3GPP}{3rd Generation Partnership Project}
\acrodef{BS}{base station}
\acrodef{DF}{decode and forward}
\acrodef{CCDF}{complementary cumulative distribution function}
\acrodef{ZF}{zero forcing}
\acrodef{RZF}{regularized zero forcing}
\acrodef{WLLN}{weak law of large number}
\acrodef{SLLN}{strong law of large numbers}
\acrodef{TDD}{Time-division duplex}
\acrodef{EE}{energy efficiency} 
\usepackage{color}
\usepackage{dsfont}
\usepackage{bbm}

\DeclareMathAlphabet{\mathsf}{OML}{cmbr}{m}{it}

\newtheorem{definition}{\bf Definition}
\newtheorem{theorem}{\bf Theorem}
\newtheorem{lemma}{\bf Lemma}

\newtheorem{assumption}{\bf Assumption}

\newcommand{\bd}{\begin{description}}
\newcommand{\ed}{\end{description}}
\newcommand{\be}{\begin{enumerate}}
\newcommand{\ee}{\end{enumerate}}
\newcommand{\bi}{\begin{itemize}}
\newcommand{\ei}{\end{itemize}}
\newcommand{\bl}{\begin{list}}
\newcommand{\el}{\end{list}}
\newcommand{\bt}{\begin{tabbing}}
\newcommand{\et}{\end{tabbing}}

  \graphicspath{{../Figures/}}
  \DeclareGraphicsExtensions{.eps}

\usepackage{amssymb}
\usepackage{amsthm}

\title{ Federated Stochastic Gradient Descent Begets Self-Induced Momentum }
\name{ Howard~H.~Yang$^\mathsection$, 
       Zuozhu~Liu$^\mathsection$, 
       Yaru~Fu$^\ast$, 
       Tony~Q.~S.~Quek$^\dagger$, 
       and H. Vincent Poor$\ddagger$ 
       \thanks{This work was supported in part by the Zhejiang Provincial Natural Science Foundation of China under Grant No. LGJ22F010001. (\textit{Corresponding Author: Zuozhu Liu.}) 
}\footnotemark[1]
}
\address{ $^\mathsection$ \textit{Zhejiang University/University of Illinois at Urbana-Champaign Institute, Haining 314400, China }\\
$^\ast$ \textit{Hong Kong Metropolitan University, Hong Kong, 999077} \\
$^\dagger$ \textit{Singapore University of Technology and Design, Singapore 487372}\\
$\ddagger$ \textit{Princeton University, Princeton, NJ 08544, USA} }
				
\begin{document}
%

\maketitle

\begin{abstract}
Federated learning (FL) is an emerging machine learning method that can be applied in mobile edge systems, in which a server and a host of clients collaboratively train a statistical model utilizing the data and computation resources of the clients without directly exposing their privacy-sensitive data. 
We show that running stochastic gradient descent (SGD) in such a setting can be viewed as adding a momentum-like term to the global aggregation process.
Based on this finding, we further analyze the convergence rate of a federated learning system by accounting for the effects of parameter staleness and communication resources.
These results advance the understanding of the Federated SGD algorithm, and also forges a link between staleness analysis and federated computing systems, which can be useful for systems designers.
\end{abstract}


\section{Introduction}
Federated learning (FL) is a branch of machine learning models that allow a computing unit, i.e., an edge server, to train a statistical model from data stored on a swarm of end-user entities, i.e., the clients, without directly accessing the clients' local datasets \cite{MaMMooRam:16}.
Specifically, instead of aggregating all the data to the server for training, FL brings the machine learning models directly to the clients for local computing, where only the resulting parameters are uploaded to the server for global aggregation, after which an improved model is sent back to the clients for another round of local training \cite{LimLuoHoa:20}. Such a training process usually converges after sufficient rounds of parameter exchanges and computing among the server and clients, upon which all the participants can benefit from a better machine learning model \cite{KonMcMBren:16,LiHuaYan:20,KhaMisRic:20}. As a result, the salient feature of on-device training mitigates many of the systemic privacy risks as well as communication overheads, hence making FL particularly relevant for next-generation mobile networks \cite{ZhaFenYan:20,ParSamBen:19,LetCheShi:19}.
Nonetheless, in the setting of FL, the server usually needs to link up a massive number of clients via a resource-limited medium, e.g., the spectrum, and hence only a limited number of the clients can be selected to participate in the federated training during each round of iteration \cite{YanLiuQue:20,YanAraQue:20ICASSP,WanTuoSal:19JSAC,CheYanSaa:19}. This, together with the fact that the time spent on transmitting the parameters can be orders of magnitude higher than that of local computations \cite{LanLeeZho:17,ArjShaSre:20}, makes the straggler issue a serious one in FL. To that end, a simple but effective approach has been proposed \cite{CheGiaSun:18}, i.e., reusing the outdated parameters in the global aggregation stage so as to accelerate the training efficiency. The gain of this scheme has been amply demonstrated via experiments while the intrinsic rationale behind it remains unclear. 
In this paper, we take the stochastic gradient descent (SGD)-based FL training as an example and show that reusing the outdated parameters implicitly introduces a momentum-like term in the global updating process, and prove the subsequent convergence rate of federated computing.
This result advances the understanding of FL and may be useful to guide further research in this area.

\vspace{-.15in}

\section{System Model}\label{sec:sysmod}
Let us consider an FL system consisting of one server and $K$ clients, as depicted per Fig.~1, where $K$ is usually a large number. Each client $k$ has a local dataset $\mathcal{D}_k = \{ \mathbf{x}_i \in \mathbb{R}^d, y_i \in \mathbb{R} \}_{ i = 1 }^{n_k}$ with size $\vert \mathcal{D}_k \vert = n_k$, and we assume the local datasets are statistically independent across the clients.
The goal of the server and clients is to jointly learn a statistical model over the datasets residing on all the clients without sacrificing their privacy. To be more concrete, the server aims to fit a vector $\mathbf{w} \in \mathbb{R}^d$ so as to minimize the following loss function without having explicit knowledge of $\mathcal{D} = \cup_{k=1}^K \mathcal{D}_k$:
\begin{align}\label{equ:Obj_Func}
\min_{ \mathbf{w} \in \mathbb{R}^d } f( \mathbf{w} ) &= \frac{1}{n} \sum_{ i=1 }^{n} \ell( \mathbf{w}; \mathbf{x}_i, y_i  ) 
\nonumber\\
&= \frac{n_k}{n} \cdot \frac{1}{n_k} \sum_{ j=1 }^{ n_k } \ell( \mathbf{w}; \mathbf{x}_j, y_j )
\nonumber\\
&= \sum_{ k=1 }^K p_k f_k( \mathbf{w} )
\end{align}
where $n = \sum_{ k=1 }^K n_k$, $p_k = n_k/n$, $\ell(\cdot)$ is the loss function assigned on each data point, and $f_k ( \mathbf{w} ) = \sum_{ j=1 }^{ n_k } \ell( \mathbf{w}; \mathbf{x}_j, y_j )/n_k$ is the local empirical loss function of client $k$.

\begin{figure}[t!]
  \centering{}

    {\includegraphics[width=0.95\columnwidth]{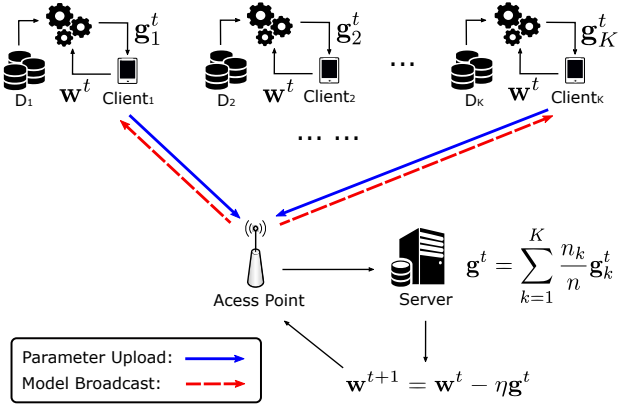}}

  \caption{ An illustration of Federated SGD training: (A) clients leverage their local datasets to evaluate the gradient term, (B) the server aggregates the received updates to produce a new global model, (C) the new model is sent back to the clients, and the process is repeated. }
  \label{fig:FL_WN}
\end{figure}

Because the server has no direct access to the individual datasets, the model training needs to be carried out by the clients in a federated fashion. In this paper, we adopt Federated SGD, a widely used mechanism, for this task. 
The details are summarized in Algorithm~\ref{Alg:Gen_FL} \cite{LimLuoHoa:20}.
Specifically, at iteration $t$, the server needs to send the global model $\mathbf{w}^t$ to a subset of clients $S_t$, where in general $N = \vert S_t \vert \ll K$ because the limited communication resources cannot support simultaneous transmissions from a vast number of clients \cite{YanLiuQue:20}, for on-device model training.
Upon receiving $\mathbf{w}^t$, the selected clients will leverage it to evaluate the gradient of the local empirical loss -- by means of an $H$-step estimation -- and upload the estimated gradients $\mathbf{g}^{t}_k$, $k \in S_t$. 
In essence, this comprises computing the stochastic gradient with a batch size of $H$ data points.
Finally, the server aggregates the collected parameters to produce a new output per \eqref{equ:wt_plain}.
Such an orchestration amongst the server and clients repeats for a sufficient number of communication rounds until the learning process converges.

It is worth noting that the gradient aggregation step \eqref{equ:gt_plain} in Algorithm~\ref{Alg:Gen_FL} utilizes not only the fresh updates collected from the selected clients but also the outdated gradients from the unselected ones. As will be shown later, this procedure, in essence, induces an implicit momentum into the learning process.

\begin{algorithm}[t!]
\caption{ Federated SGD Algorithm }
\begin{algorithmic}[1] 
\State \textbf{Parameters:} $H$ = number of local steps per communication round, $\eta$ = step size for stochastic gradient descent
\State \textbf{Initialize:} $\mathbf{w}^0 \in \mathbb{R}^d$
\For { $t = 0, 1, 2, ..., T-1$ }
\State The server randomly selects a set $S_t$ of $N$ clients and broadcasts the global parameter $\mathbf{w}^t$ to them
\For { each client $k \in S_t$ in parallel }
\State Initialize $\mathbf{g}_k^{t,0} = 0$
\For { $s$ = 0 to $H-1$ }
\State Sample $i \in \mathcal{D}_k$ uniformly at random, and update the local estimation of the gradient, $\mathbf{g}^{t,s}_k$, as follows:
\begin{align} \label{equ:LocGraDscnt}
\mathbf{g}_k^{t, s+1} = \mathbf{g}_k^{t,s} + \nabla \ell( \mathbf{w}^{t}; \mathbf{x}_i, y_i )
\end{align}
\EndFor
\State Set $\mathbf{g}_k^{t} = \mathbf{g}_k^{t, H}/H$ and send the parameter back to the server
\EndFor
\State The server collects all the updates of $\{ \mathbf{g}^{t}_i \}_{ i \in S_t }$ and assigns $\mathbf{g}^{t}_j = \mathbf{g}^{t-1}_j$ for all $j \neq S_t$. Then, the server updates both the estimation of gradient $\mathbf{g}^t$ and parameter $\mathbf{w}^{t+1}$ as follows:
\begin{align} \label{equ:gt_plain}
\mathbf{g}^{t} = \sum_{k=1}^K \frac{ n_k }{ n } \mathbf{g}^{t}_k, \\ \label{equ:wt_plain}
\mathbf{w}^{t+1} = \mathbf{w}^t - \eta \mathbf{g}^t
\end{align}
\EndFor
\State \textbf{Output:} $\mathbf{w}^T$
\end{algorithmic} \label{Alg:Gen_FL}
\end{algorithm} 

\vspace{-.15in}
\section{ Analysis }\label{sec:Analysis}
This section comprises the main technical part of this paper, in which we analytically characterize the updating process of global parameters and derive the convergence rate of the Federated SGD algorithm.

\subsection{Update Process of Global Parameters }
Due to limited communication resources, the server can only select a subset of the clients to conduct local computing and update their gradients in every round of global iteration. As a result, the gradients of the unselected clients become stale.
In accordance with \eqref{equ:gt_plain} and \eqref{equ:wt_plain}, after the $t$-th communication round, the update of global parameters at the server side can be rewritten as follows:
\begin{align} \label{equ:UpdateEqu}
\mathbf{w}^{t+1} = \mathbf{w}^t - \eta \, \sum_{ k = 1}^K p_k \, \mathbf{g}^{t-\tau_k}_k
\end{align}
in which $\tau_k$ is the staleness of the parameters corresponding to the $k$-th client.
Because the clients to participate in the FL are selected uniformly at random in each communication round, the staleness of parameters, $\{\tau_k\}_{k=1}^K$, can be abstracted as independently and identically distributed (i.i.d.) random variables with each following a geometric distribution:
\begin{align}
\mathbb{P}(\tau_k = l) = \beta^{l} ( 1 - \beta ), \quad l = 0, 1, 2, ...
\end{align}
where $\beta =  1 - N/K$.

These considerations bring us to our first result.
\begin{lemma} \label{lma:UpdateProcess}
\textit{ Under the depicted FL framework, the parameter updating process constitutes the following relationship:
  \begin{align}
  \mathbb{E}\big[ \mathbf{w}^{ t + 1 } \!-\! \mathbf{w}^t \big] \!=\! \beta\, \mathbb{E} \big[ \mathbf{w}^{ t } \!-\! \mathbf{w}^{ t - 1 } \big] \!-\! ( 1 \!-\! \beta ) \eta \, \mathbb{E}\big[ \mathbf{g}^t \big].
  \end{align}
}
\end{lemma}
\begin{proof}
Using \eqref{equ:UpdateEqu}, we can subtract $\mathbf{w}^t$ from $\mathbf{w}^{t+1}$ and obtain the following: 
\begin{align} \label{equ:SubtractForm_Explct}
\mathbf{w}^{t+1} \!-\! \mathbf{w}^{t} &=  \mathbf{w}^{t} \!-\! \mathbf{w}^{ t - 1 }
 -\! \eta \! \sum_{k=1}^K p_k \big( \mathbf{g}^{t-\tau_k}_k - \mathbf{g}^{t-\tau_k-1}_k \big).
\end{align}
By taking an expectation with respect to the staleness $\tau_k$, $k \in \{1, \cdots, K\}$ on both sides of the above equation, the following holds:
\begin{align} \label{equ:SubtractForm}
\mathbb{E}[ \mathbf{w}^{t+1} - \mathbf{w}^{t} ] &= \mathbb{E}[ \mathbf{w}^{t} - \mathbf{w}^{ t - 1 } ]
\nonumber\\
&\quad - \eta \sum_{k=1}^K p_k \underbrace{ \mathbb{E} \big[ \, \mathbf{g}^{t-\tau_k}_k - \mathbf{g}^{t-\tau_k-1}_k \, \big] }_{Q_1}.
\end{align}
Since $\tau_k \sim Geo( 1 - \beta )$, we can calculate $Q_1$ as
\begin{align}
Q_1
&=\! \big( 1 - \beta \big) \mathbb{E}\big[ \mathbf{g}^t_k \big] + \sum_{ l = 1 }^\infty  (1 - \beta)\beta^{l} \mathbb{E}\big[ \mathbf{g}^{t-l-1}_k \big]
\nonumber\\
&~~\qquad \qquad \qquad \qquad \qquad - \sum_{ l = 0}^\infty (1-\beta) \beta^{l} \mathbb{E}\big[ \mathbf{g}^{t-l-1}_k \big]
\nonumber\\
&=\! \big( 1 - \beta \big) \mathbb{E}\big[ \mathbf{g}^{t}_k \big]  - \sum_{ l = 0}^\infty (1 \!-\! \beta)^2 \beta^{l} \mathbb{E}\big[ \mathbf{g}^{t-l-1}_k \big].
\end{align}
Furthermore, by noticing that for the stochastic gradient of each client $k$, the following result holds:
\begin{align}
 \mathbb{E} \big[ \mathbf{g}_k^{ t - \tau_k - 1} \big] = \sum_{ l = 0}^\infty\,  ( 1 - \beta) \beta^{l} \mathbb{E}\big[ \mathbf{g}_k^{ t - l - 1} \big],
\end{align}
we have
\begin{align} \label{equ:ExpanVersion}
& \eta \sum_{k=1}^K p_k \, \mathbb{E} \big[\, \mathbf{g}_k^{ t - \tau_k } - \mathbf{g}_k^{ t - \tau_k - 1} \, \big]
\nonumber\\
&=\! \big( 1 \!-\! \beta \big) \eta \sum_{k=1}^K p_k \mathbb{E}\big[ \mathbf{g}_k^t \big]
\nonumber\\
&\qquad \qquad \qquad \!-\! \big( 1 \!-\! \beta \big) \eta\! \sum_{k=1}^K p_k \! \sum_{l=0}^\infty  (1 \!-\! \beta) \beta^{l} \mathbb{E} \big[ \mathbf{g}_k^{ t - l -1 } \big]
\nonumber\\
&= \! \big( 1 \!-\! \beta \big) \eta \, \mathbb{E}\big[ \mathbf{g}^t \big]
 \!-\! \big( 1 \!-\! \beta \big) \eta\! \sum_{k=1}^K p_k  \mathbb{E} \big[ \mathbf{g}_k^{ t - \tau_k -1 } \big]
\nonumber\\
&\stackrel{(a)}{=}  \! \big( 1 \!-\! \beta \big) \eta \, \mathbb{E}\big[ \mathbf{g}^t \big] + \big( 1 \!-\! \beta \big) \mathbb{E} \big[ \mathbf{w}^{ t } \!-\! \mathbf{w}^{ t - 1 } \big],
\end{align}
where ($a$) follows from \eqref{equ:UpdateEqu}.
Finally, by substituting \eqref{equ:ExpanVersion} into \eqref{equ:SubtractForm}, we complete the proof.
\end{proof}

From Lemma~\ref{lma:UpdateProcess}, we can identify a momentum-like term, namely $\beta \mathbb{E}[ \mathbf{w}^t - \mathbf{w}^{t-1} ]$, when the global parameter is updated from $\mathbf{w}^t$ to $\mathbf{w}^{t+1}$.
This can mainly be attributed to the 
reuse of gradients, which introduces memory during the global aggregation step and makes the parameter vector $\mathbf{w}^{t+1}$ stay close to the current server model $\mathbf{w}^{t}$.
Notably, such a phenomenon is also observable in the context of completely asynchronized SGD algorithms owing to similar reasons \cite{MitZhaHad:16}.
As a result, Lemma~\ref{lma:UpdateProcess} can serve as a useful reference to adjust the controling factor if one intends to accelerate the Federated SGD algorithm by running it in conjunction with an \textit{explicit momentum} term \cite{MitZhaHad:16,LiuCheChe:20,HuoYanGu:20}.
Besides, if the delayed gradient averaging such as \cite{ZhuLinLu:21} is employed, the design of \textit{gradient correction} shall take into account the effect of such an implicit momentum as well. 

In the sequel, we quantify the effect of this implicit momentum on the convergence performance of the FL system.

\vspace{-.15in}

\subsection{ Convergence Analysis }
To facilitate the analysis of the FL convergence rate, we make the following assumption on the structure of the global empirical loss function.
\begin{assumption}
\textit{The gradient of each $f_k$ is Lipschitz continuous with a constant $L>0$, i.e., for any $\mathbf{w}, \mathbf{v} \in \mathbb{R}^d$ the following is satisfied:
  \begin{align}
  \Vert \nabla f_k( \mathbf{w} ) - \nabla f_k( \mathbf{v} ) \Vert \leq L \Vert \mathbf{w} - \mathbf{v} \Vert.
  \end{align}
}
\end{assumption}

This assumption is standard in the machine learning literature and is satisfied by a wide range of machine learning models, such as SVM, logistic regression, and neural networks. Besides, no assumption regarding the convexity of the objective function is made.
We further leverage a notion, termed gradient coherence, to track the variant of the gradient during the training process, defined as follows \cite{DaiZhoDon:19}.
\begin{definition}
\textit{ The gradient coherence at communication round $t$ is defined as
\begin{align}
\mu_t = \min_{ 0 \leq s \leq t } \frac{ \langle\, \nabla f( \mathbf{w}^s ) , \nabla f(\mathbf{w}^t) \,\rangle }{ \Vert \nabla f( \mathbf{w}^s ) \Vert^2 }.
\end{align}
}
\end{definition}

The gradient coherence characterizes the largest deviation of directions between the current gradient and the gradients along the past iterations. As such, if $\mu_t$ is positive, then the direction of the current gradient is well aligned to those of the previous ones, and hence reusing the trained parameters can push forward the global parameter vector toward the optimal point.

\begin{theorem} \label{thm:ConvRate}
\textit{Suppose the gradient coherence $\mu_t$ is lower bounded by some $\mu > 0$ for all $t$ and the variance of the stochastic gradient is upper bounded by $\sigma^2 > 0$. If we choose the step size to be $\eta = 1/\sqrt{LT}$, then after $T$ rounds of communication, the Alg.~1 converges as follows:
\begin{align}
\min_{ 0 \leq t \leq T-1 } \!\! \mathbb{E}\big[ \Vert \nabla f( \mathbf{w}^t ) \Vert^2 \big] \leq \frac{ 2 \sqrt{L} \big[  f(\mathbf{w}^0) -   f( \mathbf{w}^* ) + \sigma^2 \big] }{ \big[ 1 - ( 1 - \mu ) \beta \big] \sqrt{T} }
\end{align}
in which $\mathbf{w}^* = \arg\min_{\mathbf{w} \in \mathbb{R}^d } f(\mathbf{w})$.
 }
\end{theorem}
\begin{proof}
Following Assumption~1, we know that the empirical loss function $f(\cdot)$ is $L$-smooth, and hence after the $t$-th round of global parameter update the following holds:
\begin{align} \label{equ:Ef_wt}
& \mathbb{E}\big[ f( \mathbf{w}^{t+1} ) \big] \leq \mathbb{E}\big[ f( \mathbf{w}^t ) \big] + 
\underbrace{ \mathbb{E}\big[ \langle\, \mathbf{w}^{t+1} - \mathbf{w}^t, \nabla f(\mathbf{w}^t) \,\rangle \big] }_{Q_1}
\nonumber\\
&\qquad \qquad \qquad  + \frac{ L }{ 2 } \underbrace{ \mathbb{E}\big[ \Vert \mathbf{w}^{t+1} - \mathbf{w}^t \Vert^2 \big] }_{Q_2}.
\end{align}
Using Lemma~\ref{lma:UpdateProcess}, we can expand the terms in $Q_1$ and obtain the following:
\begin{align} \label{equ:TermQ1}
& Q_1\stackrel{(8)}{=} \mathbb{E}\big[ \langle\, \beta (\mathbf{w}^{t} - \mathbf{w}^{t-1}) - \eta (1-\beta) \mathbf{g}^t, \nabla f(\mathbf{w}^t) \,\rangle \big]
\nonumber\\
&\stackrel{(a)}{=} \beta \mathbb{E}\big[ \langle\, \mathbf{w}^{t} - \mathbf{w}^{t-1}, \nabla f(\mathbf{w}^t) \,\rangle \big] - \eta (1-\beta) \mathbb{E}\big[ \Vert \nabla f( \mathbf{w}^t ) \Vert^2 \big]
\nonumber\\
&\stackrel{(b)}{=} \!- \eta \beta \mathbb{E}\big[ \langle\, \nabla f( \mathbf{w}^{t-\tau-1} ) , \nabla f(\mathbf{w}^t) \,\rangle \big] \!-\! \eta (1 \!-\! \beta) \mathbb{E} \big[ \Vert \nabla f( \mathbf{w}^t ) \Vert^2 \big]
\nonumber\\
&\leq  - \eta \beta \mu  \mathbb{E} \big[ \Vert \nabla f( \mathbf{w}^t ) \Vert^2 \big] - \eta (1-\beta) \mathbb{E} \big[ \Vert \nabla f( \mathbf{w}^t ) \Vert^2 \big]
\nonumber\\
&= - \eta \big[ 1 - (1-\mu) \beta \big] \mathbb{E} \big[ \Vert \nabla f( \mathbf{w}^t ) \Vert^2 \big]
\end{align}
where ($a$) follows by noticing that $\mathbb{E}[\mathbf{g}^t] = \nabla f(\mathbf{w}^t)$ and in ($b$) we notice that all the random variables $\{ \tau_k \}_{k=1}^K$ possess the same distribution and hence unify them by introducing a random variable $\tau$ that satisfies $\tau = \tau_k$ in distribution. On the other hand, as the stochastic gradient has a bounded variance, $Q_2$ can be evaluated as
\begin{align} \label{equ:TermQ2}
Q_2 &= \mathbb{E}\Big[ \big\Vert \eta \sum_{k=1}^K p_k \mathbf{g}_k^{t-\tau_k} \big \Vert^2 \Big]
\nonumber\\
&= \eta^2 \mathbb{E}\Big[ \big\Vert \sum_{k=1}^K p_k \mathbf{g}_k^{t-\tau_k} - \nabla f( \mathbf{w}^t ) + \nabla f( \mathbf{w}^t ) \big \Vert^2 \Big]
\nonumber\\
&\leq 2 \eta^2 \big(  \mathbb{E}\big[ \Vert \nabla f(\mathbf{w}^{t}) \Vert^2 \big] + \sigma^2 \big).
\end{align}

By taking \eqref{equ:TermQ1} and \eqref{equ:TermQ2} back into \eqref{equ:Ef_wt} and telescoping $t$ from 0 to $T-1$, we have
\begin{align}
&\mathbb{E}[f( \mathbf{w}^{T-1} )] - \mathbb{E} [ f( \mathbf{w}^0 ) ] \leq L \sum_{ t = 0 }^{T-1} \eta^2 \mathbb{E}\big[ \Vert \nabla f( \mathbf{w}^t ) \Vert^2 \big]
\nonumber\\
&  + \sum_{ t = 0}^{T-1} { L \eta^2 \sigma^2 } \!-\!  \big( 1 \!-\! (1 \!-\! \mu) \beta \big) \sum_{ t = 0 }^{T-1} \! \eta \, \mathbb{E}\big[ \Vert \nabla f(\mathbf{w}^t) \Vert^2 \big].
\end{align}

Further rearranging the terms of the above inequality yields
\begin{align}
&\sum_{t=0}^{T-1} \big[ \big( 1 - ( 1-\mu ) \beta \big) \eta - {L} \eta^2 \big] \mathbb{E}\big[ \Vert f( \mathbf{w}^t ) \Vert^2 \big]
\nonumber\\
&\leq  f( \mathbf{w}_0 ) - \mathbb{E}\big[ f( \mathbf{w}^{T-1} ) \big] + { L \sigma^2 } \sum_{ t = 1 }^T \eta^2.
\end{align}
Note that $\mathbb{E}[ f(\mathbf{w}^T) ] \geq f(\mathbf{w}^*)$ and $\eta = 1/\sqrt{LT}$, and so we have
\begin{align}
& \min_{ 0 \leq t \leq T - 1 } \!\!\! \mathbb{E}\big[ \Vert \nabla f( \mathbf{w}^t )  \Vert^2 \big] \leq \frac{ f( \mathbf{w}_0 ) - f(\mathbf{w}^*) + { L \sigma^2 } \sum_{ t = 1 }^{T-1} \eta^2 }{ \sum_{ t = 0 }^{T-1} \big[ \big( 1 - ( 1 - \mu ) \beta \big) \eta - {L} \eta^2 \big] }
\nonumber\\
& = \frac{ f( \mathbf{w}_0 ) - f(\mathbf{w}^*) + \sigma^2 }{ \big( 1 - ( 1 - \mu ) \beta \big) \sqrt{T/L} - 1 }.
\end{align}
Finally, when $T$ is taken to be sufficiently large, we have 
\begin{align}
\big( 1 - ( 1 - \mu ) \beta \big) \sqrt{T/L} - 1 \geq \frac{ \big( 1 - ( 1 - \mu ) \beta \big) \sqrt{T} }{ 2 \sqrt{L} }
\end{align}
and the result follows.
\end{proof}

Following Theorem~\ref{thm:ConvRate}, several observations can be made: ($i$) For non-convex objective functions, Federated SGD converges to stationary points on the order of $1/\sqrt{T}$; ($ii$) the staleness of parameters impacts the convergence rate via the multiplicative constant, which unveils that when the communication resources are abundant, i.e., the server can select many clients for parameter updates in each iteration, that leads to an increase in $N$ which in turns reduces $\beta$ and results in a faster convergence rate, and vice versa; and
($iii$) this result also provides further evidence to the claim that having more clients participate in each round of FL training is instrumental in speeding up the model convergence \cite{YanJiaShi:20,NisYon:19}.{\footnote{A few simulation examples that corroborate these observations are available in: https://person.zju.edu.cn/person/attachments/2022-01/01-1641711371-850767.pdf }}

\vspace{-.15in}
\section{conclusion}

In this paper, we have carried out an analytical study toward a deeper understanding of the FL system.
For the Federated SGD algorithm that uses both fresh and outdated gradients in the aggregation stage, we have shown that this implicitly introduces a momentum-like term during the update of global parameters. We have also analyzed the convergence rate of such an algorithm by taking into account the parameter staleness and communication resources.
Our results have confirmed that increasing the number of selected clients in each communication round can accelerate the convergence of the FL algorithm through a reduction in the staleness of parameters.
The analysis does not assume convexity of the objective function and hence is applicable to even the setting of deep learning systems.
The developed framework reveals a link between staleness analysis and FL convergence rate, and may be useful for further research in this area.
%

\bibliographystyle{IEEEbib}
\bibliography{bib/StringDefinitions,bib/IEEEabrv,bib/howard_FedSGD_Momentum}
\end{document}